\documentclass[sn-mathphys]{sn-jnl}

\usepackage{amsmath}
\usepackage{amsfonts}
\usepackage{subfig}
\graphicspath{{figures/}}

\jyear{2022}%

\theoremstyle{thmstyleone}%
\newtheorem{theorem}{Theorem}
\newtheorem{proposition}[theorem]{Proposition}%

\newtheorem{lemma}[theorem]{Lemma}%
\newtheorem{corollary}[theorem]{Corollary}%

\def\mc#1{\mathcal{#1}}
\def\mb#1{\mathbf{#1}}
\def\diag{\mathrm{diag}}
\def\mbb#1{\mathbb{#1}}

\begin{document}

\title[Stability and Generalization of Hypergraph Collaborative Networks]{Stability and Generalization of Hypergraph Collaborative Networks}


\author[1,2]{\fnm{Michael K.} \sur{Ng}}

\author[3]{\fnm{Hanrui} \sur{Wu}}

\author[2]{\fnm{Andy} \sur{Yip}}

\affil[1]{\orgdiv{Institute of Data Science}, \orgname{The University of Hong Kong}, \orgaddress{\street{Pokfulam Road}, \city{Hong Kong},  \country{China}}}

\affil[2]{\orgdiv{Department of Mathematics}, \orgname{The University of Hong Kong}, \orgaddress{\street{Pokfulam Road}, \city{Hong Kong},  \country{China}}}

\affil[3]{\orgdiv{College of Information Science and Technology},
\orgname{Jinan University}, \orgaddress{\city{Guangzhou}, \country{China}, \postcode{510006}}}


\abstract{Graph neural networks have been shown to be very effective in utilizing pairwise relationships across samples.
Recently, there have been several successful proposals to generalize graph neural networks to hypergraph neural networks
to exploit more complex relationships. In particular, the hypergraph collaborative networks
yield superior results compared to other hypergraph neural networks for various semi-supervised learning tasks.
The collaborative network can provide high quality vertex embeddings and hyperedge embeddings
together by formulating them as a joint optimization
problem and by using their consistency in reconstructing the given hypergraph. In this paper, we aim
to establish the algorithmic stability of the core layer of the collaborative network
and provide generalization guarantees.
The analysis sheds light on the design of hypergraph filters in collaborative networks,
for instance, how the data and hypergraph filters
should be scaled to achieve uniform stability of the learning process.
Some experimental results on real-world datasets are presented to illustrate the theory.
}




\maketitle

\section{Introduction}\label{sec:intro}

Many real-world applications involve datasets that exhibit graph structures that depict pairwise relationships between
vertices. These applications span a wide range of domains, including text analysis \cite{yao-mao-luo18,xu-etal:20},
social network analysis \cite{chen-li-bruna20}, molecule classification in chemistry \cite{gilmer-etal17},
point cloud processing \cite{te-etal18}, mesh generation \cite{litany-etal18}, and knowledge graphs \cite{rui-etal:22}, to name a few.
In some applications, the only available data are graphs, while in some other applications, features about the
vertices and edges are also available. Very naturally, machine learning algorithms that can exploit both graphs and
features often yield better results than algorithms that solely rely on graph structures.
Of particular interest is the class of graph convolution network (GCN) methods \cite{kipf-welling17, li-etal18, such-etal17}.
These methods use neural network layers whose output at a vertex
depends mostly on others that are deemed relevant according to the graph structure. A very common point of view is that
GCN is a generalization of the traditional spatial filters, where graphs can represent neighbors that are not limited to
spatial closeness.

While graphs can depict pairwise relationships, hypergraphs can represent relationships between multiple vertices.
A hypergraph is a graph in which each edge, or hyperedge, can connect to multiple vertices. Hypergraphs can therefore
represent more complex relationships. For example, in traditional paper-authorship networks, two articles (vertices) are
connected if they share one or more coauthors. In this way, authorship information, which can provide important clues
to the topic of the articles, is lost. Hypergraphs come to the rescue by treating each author as a hyperedge and each article as a vertex.
Several works have been devoted to hypergraph learning. In \cite{zhou-huang-scholkopf06},
spectral clustering and semi-supervised parametric models on graphs were extended to hypergraphs.
In \cite{gao-etal20}, a tensor representation of hypergraphs was proposed to make the optimization of hypergraphs more amenable.
Generalizations of graph convolution networks to hypergraph convolution networks have been studied in
\cite{feng-etal19, bai-zhang-torr21, dong-sawin-bengio20, yadati-etal19}.
The main task is to define
a convolution operator in a hypergraph such that
the transition probability between two vertices can be measured, and
the embeddings (or features) of each vertex can be propagated in a
hypergraph convolution network. More propagations should
be done between vertices connected by a common hyperedge.

Recently, \cite{wu-ng22} and \cite{wu-yan-ng22} studied and developed
the convolution of vertex and hyperedge features
to suitably aggregate values in hypergraphs. It was shown
that such hypergraph collaborative networks (HCoN)
have obtained superior results in some semi-supervised learning problems compared to
other hypergraph neural network models.
The proposal is to formulate the learning of vertex and hyperedge embeddings as
a joint optimization problem to allow for updating the vertex and hyperedge embeddings simultaneously.
The authors showed that the performance can be further boosted by incorporating a hypergraph reconstruction error as one of the objectives.

Apart from the development of machine learning algorithms and applications, efforts on the theory
have also been made in the past decades, from general learning theory to specific properties of a class of algorithms.
In this paper, we aim to study the algorithmic stability and generalization of the class of HCoN. We show that,
in the single-layer case, the HCoN is algorithmically stable. As a result, the generalization
gap (difference between training error and testing error) converges to zero as the size of the training set increases.
We should explain the notions of training and testing in our context in Section \ref{sec:sgd} below.
The analysis sheds light on the design of hypergraph filters in HCoNs, for instance, how the data and filters
should be scaled to achieve the uniform stability of the learning process. Our
generalization result is also valid for hypergraph convolution networks with
the propagation of vertex embeddings only.

The rest of the paper is organized as follows. In Section 2, we review some works
related to hypergraph neural networks
and generalization guarantee studies. In Section 3, we introduce the HCoN model, some assumptions about
the loss function and activation function, and some preliminary results.
The results of the algorithmic stability of the HCoN and the generalization guarantee are established in Section 4.
In Section 5, we present some experimental studies on the generalization gap to illustrate the theory.
Some concluding remarks are given in Section 6.

\section{Related Work}

Generalization guarantees concern the expected difference between training and testing errors.
Bousquet and Elisseeff \cite{bousquet-elisseeff02} and Mukherjee {\it et~al.} \cite{mukherjee-etal06} showed that
under suitable regularity assumptions, the stability of an algorithm implies generalization.
In addition to the general theory, Bousquet and Elisseeff \cite{bousquet-elisseeff02} also studied the stability and generalization
of the \emph{global minimizer} of some regularized learning models. Such results are therefore independent of the particular
algorithm used to compute an approximate optimizer.
Stochastic gradient descent (SGD) is a popular algorithm to obtain
a suboptimal solution to optimization problems. Some classical results on the generalization of SGD in the case of
a single pass of data are reported in \cite{nemirovski-yudin83}. The analysis in the case of multiple passes
of SGD is reported in \cite{hardt-recht-singer15}.

Stability and generalization analysis of SGD applied to graph convolution networks (GCNs) is studied in \cite{verma-zhang19}.
The stability is established in terms of the dominant eigenvalue of the graph filter.
A difference between \cite{verma-zhang19} and \cite{bousquet-elisseeff02} is that the former considers
the model parameters computed via SGD, whereas the latter considers the theoretical optimal model parameters.
Our work generalizes the analysis of GCNs in \cite{verma-zhang19} to hypergraph collaborative networks \cite{wu-ng22, wu-yan-ng22}
which also include some edge features. A technical difference between our work and \cite{verma-zhang19} is
that the latter assumes some Lipschitz conditions on the composition of the loss function and the neural network outputs
but we require some Lipschtiz conditions on the loss function only. Ours are therefore easier to verify and more fundamental.

Various kinds of stability have been proposed and studied in the literature \cite{bousquet-elisseeff02, mukherjee-etal06}.
Similar to \cite{verma-zhang19}, we use uniform stability which yields tighter bounds than other forms of stability,
such as error stability, hypothesis stability and pointwise hypothesis stability \cite{bousquet-elisseeff02}.
Finally, another related work is \cite{belkin-matveeva-niyogi04}, which studied regularized graph learning problems
and devised some generalization guarantees. However, the resulting generalization gap is inversely proportional to the
second smallest eigenvalue of the graph Laplacian matrix. Thus the bound can grow with the size of the graph.
The bounds devised by us and \cite{verma-zhang19} are independent of the graph size, and the generalization gap
converges to zero with the sample size and the graph size.

\section{Hypergraph Convolution Networks}

\subsection{Hypergraphs}

A hypergraph is a graph where a hyperedge can connect to any number of vertices.
In the undirected version, each hyperedge is represented by a subset of vertices. A hypergraph is denoted
by $\mc{G}=(\mc{V}, \mc{E})$ where $\mc{V}$ is the set of vertices and $\mc{E}=\{e\;:\;e\subset \mc{V}\}$ is
the set of hyperedges.
The number of vertices and the number of hyperedges are denoted by $N=\vert\mc{V}\vert$
and $M=\rvert\mc{E}\rvert$, respectively.
A hypergraph can also be represented by an incidence matrix $H\in\mathbb{R}^{N\times M}$, whose $(i,j)$-th
equals $1$ if the $j$-th hyperedge is connected to the $i$-th vertex, and equals $0$ otherwise.
For simplicity, we consider unweighted graphs, but the weights can be introduced into the hypergraph filters
easily as done in \cite{feng-etal19, bai-zhang-torr21, dong-sawin-bengio20, yadati-etal19}.

Generalizations of graph convolution networks to hypergraph convolution networks have been studied.
Given a feature matrix $X_{V}\in\mathbb{R}^{N\times F_V}$
(where $F_V$ is the dimension of the vertex features) and the dimension $O$ of the output embeddings,
the convolution operator in a hypergraph for the propagation of
vertex features (or embeddings) is constructed as follows:
\begin{equation} \label{vertexiterative}
f(X_{V} \vert Q_{V} ) =
\sigma\left(D_{V}^{-\frac{1}{2}} H D_{E}^{-1}  H^{\top}
D_{V}^{-\frac{1}{2}} X_{V} Q_{V} \right),
\end{equation}
where
$\sigma(\cdot)$ is an activation function,
$Q_{V}\in\mathbb{R}^{F_V\times O}$ is a matrix of learnable parameters,
$H\in\mathbb{R}^{N\times M}$ is the incident matrix of the hypergraph, and
$D_{V} = \diag(H\mb{1})$ and $D_{E} = \diag(H^{\top}\mb{1})$ are
diagonal matrices with the degrees of the vertices and the degrees of the hyperedges on the diagonals.
We note in \eqref{vertexiterative} that the operator
$D_{V}^{-\frac{1}{2}} H D_{E}^{-1}  H^{\top} D_{V}^{-\frac{1}{2}}$ mixes the feature vectors
of vertices connected by a common hyperedge. In this way, feature vectors of neighboring vertices of a vertex
are propagated into the vertex and aggregated to form a new feature vector for the vertex.
This model does not make use of hyperedge features.

\subsection{The Hypergraph Collaborative Network Model}\label{sec:hcon}

Unlike the hypergraph convolution networks in \eqref{vertexiterative},
the hypergraph collaborative networks (HCoN) in \cite{wu-ng22} and \cite{wu-yan-ng22}
aim to predict vertex-level and hyperedge-level labels together
based on the hypergraph structure and a set of vertex features and hyperedge features.
The single-layer hypergraph collaborative network we consider is given by
\setlength{\arraycolsep}{0.1em}
\begin{eqnarray}
f(X_{V}, X_{E} \vert \theta) &=& \sigma\left(\tilde{H} \tilde{H}^{\top} X_{V} Q_{V}
+ \tilde{H}D_{E}^{-\frac{1}{2}}X_{E}Q_{E}\right) \label{eqn:f}\\ 
g(X_{V}, X_{E} \vert \theta) &=& \sigma\left(\tilde{H}^{\top} \tilde{H} X_{E} P_{E}
+ \tilde{H}^{\top}D_{V}^{-\frac{1}{2}}X_{V}P_{V}\right). \label{eqn:g} 
\end{eqnarray}
\setlength{\arraycolsep}{5pt}%
In addition to the notations in \eqref{vertexiterative},
here $X_{E}\in\mathbb{R}^{M\times F_E}$ is a given matrix of hyperedge features,
$F_E$ is the dimension of the hyperedge features,
$P_{V}\in\mathbb{R}^{F_V\times O}$,
$Q_{E}\in\mathbb{R}^{F_E\times O}$ and
$P_{E}\in\mathbb{R}^{F_E\times O}$ are the additional learnable parameters,
$\theta=(Q_{V},Q_{E},P_{V},P_{E})$, and
$\tilde{H}=D_{V}^{-\frac{1}{2}} H D_{E}^{-\frac{1}{2}}$ is the normalized incident matrix.
The functions $f$ and $g$ are referred to as the vertex encoder and the hyperedge encoder, respectively.
They are used to produce labels at vertices and hyperedges, respectively.
In the sequel, we focus on the analysis of vertex encoder \eqref{eqn:f}
because the analysis of \eqref{eqn:g} is essentially the same.

We remark that the HCoN model described in \cite{wu-yan-ng22} reads
\setlength{\arraycolsep}{0.1em}
\begin{eqnarray}
f(X_{V}, X_{E} \vert \theta) &=& \sigma\left(\alpha\tilde{H} W\tilde{H}^{\top} X_{V} Q_{V}
+ (1-\alpha)\tilde{H} WD_{E}^{-\frac{1}{2}}X_{E}Q_{E}\right) \label{eqn:f_weighted}\\
g(X_{V}, X_{E} \vert \theta) &=& \sigma\left(\beta\tilde{H}^{\top} U\tilde{H} X_{E} P_{E} +
(1-\beta)\tilde{H}^{\top} UD_{V}^{-\frac{1}{2}}X_{V}P_{V}\right). \label{eqn:g_weighted}
\end{eqnarray}
\setlength{\arraycolsep}{5pt}%
In this work, we assume that the vertex weights $W$ and hyperedge weights $U$ are set to identity matrices for ease of presentation.
We also absorb the constants $\alpha$ and $\beta$ into the parameters $Q_{V}$, $Q_{E}$, $P_{V}$, and $P_{E}$; the degree of freedom
of the model remains unchanged.

\subsection{The Activation Function}

Algorithmic stability concerns the change in the loss function value with respect to the change in the data.
Since SGD aggregates gradients, it is very natural that the stability must rely on the regularity of the
activation function and the loss function.
The activation function $\sigma:\mathbb{R}\to\mathbb{R}$ is assumed to satisfy the following.
Standard activations such as sigmoid, ELU, and tanh verify these assumptions.
RELU fails the $\sigma$-smoothness. However, one can consider a smoothed RELU to restore the theory:
$\sigma(x)=0$ if $x<-\epsilon$, $(x+\epsilon)^2/(2\epsilon)^2$ if $-\epsilon \le x \le \epsilon$, $x$ if $x>\epsilon$.
\begin{enumerate}
\item $\sigma$ continuous differentiable

\item $\sigma$-Lipschitz:
\begin{equation}\label{eqn:sigma}
 \vert \sigma(x) - \sigma(y) \vert  \le \alpha_{\sigma} \vert x-y \vert  \qquad\mbox{for all }x,y\in\mathbb{R}.
\end{equation}
It follows that the derivative $\sigma'$ is bounded, i.e. $ \vert \sigma'(x) \vert  \le \alpha_{\sigma}$.

\item $\sigma$-smooth:
\begin{equation}\label{eqn:sigma_prime}
 \vert \sigma'(x) - \sigma'(y) \vert  \le \nu_{\sigma} \vert x-y \vert  \qquad\mbox{for all }x,y\in\mathbb{R}.
\end{equation}
\end{enumerate}

\subsection{The Loss Function}

Let $\hat{y}$ be an estimated label and let $y$ be the true label. The loss function
$\ell:[y_{\min}, y_{\max}]\times[y_{\min}, y_{\max}]\to\mathbb{R}^+$ is denoted by $\ell(\hat{y}, y)$.
The following assumptions are made.
\begin{enumerate}

\item Continuous w.r.t. $(\hat{y},y)$ and continuously differentiable w.r.t. $\hat{y}$

\item $\ell$-Lipschitz w.r.t. $\hat{y}$:
\begin{equation}\label{eqn:loss}
\left \vert \ell(\hat{y}, y) - \ell(\hat{y}', y)\right \vert  \le \alpha_{\ell} \vert \hat{y} - \hat{y}' \vert
\end{equation}
for all $\hat{y},\hat{y}',y\in[y_{\min}, y_{\max}]$.
This implies that $ \vert \frac{\partial \ell}{\partial \hat{y}}(\hat{y}, y) \vert  \le \alpha_{\ell}$.

\item $\ell$-smooth w.r.t. $\hat{y}$:
\begin{equation}\label{eqn:loss_prime}
\left \vert \frac{\partial\ell}{\partial \hat{y}}(\hat{y}, y) - \frac{\partial\ell}{\partial \hat{y}}(\hat{y}', y)\right \vert  \le \nu_{\ell} \vert \hat{y}-\hat{y}' \vert
\end{equation}
for all $\hat{y},\hat{y}',y\in[y_{\min}, y_{\max}]$.
\end{enumerate}
By the nonnegativity and continuity of $\ell(\cdot,\cdot)$ and the compactness of $[y_{\min}, y_{\max}]\times[y_{\min}, y_{\max}]$, the loss function is bounded:
\begin{equation}\label{eqn:loss_bound}
0 \le \ell(\hat{y}, y) \le \gamma_{\ell} \qquad\mbox{for all }\hat{y},y\in[y_{\min}, y_{\max}].
\end{equation}
In the case of the binary cross-entropy $\ell(\hat{y},y) = -y\ln \hat{y}$, the Lipschitzity condition does not hold.
However, it can be easily remedied by rescaling or clipping $\hat{y}$ to the range of $[\epsilon,1]$
for a $0< \epsilon < 1$, so that $\ell(\hat{y},y) \le -\ln\epsilon$.

We remark that the following Lipschitzity of a composite function is assumed in \cite{verma-zhang19}:
\begin{eqnarray*}
\left\|\nabla_{\theta}[\ell(f(v \vert \theta), y)] - \nabla_{\theta}[\ell(f(v \vert \theta'), y)]\right\|_2
&\le& \alpha_{\ell}\|\nabla_{\theta}f(v \vert \theta) - \nabla_{\theta}f(v \vert \theta')\|_2.
\end{eqnarray*}
Here, $\|\cdot\|_2$ denotes the Euclidean norm and $f(v\vert\theta)$ denotes the output of a neural network with
parameters $\theta$ at vertex $v$. The $\theta$ and $\theta'$ are two sets of parameters for the network.
The $\nabla_{\theta}$ denotes the gradient operator w.r.t. $\theta$ (i.e. $\frac{\partial}{\partial\theta}$).
Our assumption \eqref{eqn:loss_prime} is solely in terms of
$\frac{\partial\ell}{\partial \hat{y}}$ and is more fundamental.

\subsection{Network Outputs at a Single Vertex}

In this subsection, we provide more details about the characteristics of the network outputs at each vertex and
devise a bound of the hypergraph filter outputs that will be useful in the next section.
We remark that a similar analysis can also be conducted for network outputs at each hyperedge.
Here, we focus on the analysis of vertex propagation \eqref{eqn:f} only.

For a vertex $v\in\mc{V}$, let $\mb{e}(v)$ be a binary vector whose $j$-th entry is $1$ if $v$ is the $j$-th vertex
and is $0$ otherwise. Denote the learnable parameters with $\theta=(Q_{V}^{\top}, Q_{E}^{\top})^{\top}\in\mathbb{R}^{(F_V+F_E)\times O}$.
The output \eqref{eqn:f} for a vertex $v$ is given by
\begin{eqnarray*}
f(v \vert \theta) &=& \sigma\left(A_{v}Q_{V} + B_{v}Q_{E}\right) \in\mathbb{R}^{1\times O},
\end{eqnarray*}
where $A_{v}:=\mb{e}(v)^{\top}\tilde{H} \tilde{H}^{\top} X_{V}$ and
$B_{v}:=\mb{e}(v)^{\top}\tilde{H}D_{E}^{-\frac{1}{2}}X_{E}$.
Note that $A_{v}$ is a linear combination of the vertex feature vectors over the neighboring vertices of $v$
and $B_{v}$ is a linear combination of the edge feature vectors over the hyperedges joining $v$.
For notational simplicity, we assume that the output dimension is $O=1$ .
The analysis also works for a general $O$ with minor modifications.
Let
$$
d_{v} = A_{v}Q_{V} + B_{v}Q_{E}.
$$
Then, we have
\begin{eqnarray}
f(v \vert \theta) &=& \sigma(d_{v}) \in \mathbb{R} \label{eqn:f_d}\\
\nabla_{\theta} f(v \vert \theta) &=& \sigma'(d_{v})
\left(A_{v} \; B_{v}\right)^{\top} \in\mathbb{R}^{(F_V+F_E)\times 1}.\label{eqn:grad_f}
\end{eqnarray}
Here, $\nabla_{\theta}$ is the gradient operator w.r.t. $\theta$ and $\sigma'$ is the derivative of $\sigma$.

Denote by $\|\cdot\|_2$ the matrix $2$-norm. Note that
\setlength{\arraycolsep}{0.1em}%
\begin{eqnarray*}
\left\|
\left(A_{v} \; B_{v}\right)
\right\|_2^2
&=& \left\|\mb{e}(v)^{\top} \tilde{H}\tilde{H}^{\top}X_{V}\right\|_2^2
+ \left\|\mb{e}(v)^{\top} \tilde{H} D_{E}^{-\frac{1}{2}}X_{E}\right\|_2^2\\
&\le& \|X_V\|_2^2\|\tilde{H}\tilde{H}^{\top}\|_2^2 + \|X_E\|_2^2\|\tilde{H}\|_2^2\| D_{E}^{-\frac{1}{2}}\|_2^2\\
&\le& \|X_V\|_2^2\;\mu(\tilde{H})^4 + \|X_E\|_2^2\;\mu(\tilde{H})^2,
\end{eqnarray*}
\setlength{\arraycolsep}{5pt}%
where $\mu(\tilde{H})=\|\tilde{H}\|_2$ is the largest singular value of $\tilde{H}$.
The last inequality follows from the fact that the diagonal entries of $D_{E}$ are positive integers
so that $\| D_{E}^{-\frac{1}{2}}\|_2\le 1$.
Let
\begin{equation}\label{eqn:g_max}
g_{\max} := \mu(\tilde{H})\sqrt{\|X_V\|_2^2 \;\mu(\tilde{H})^2 + \|X_E\|_2^2}.
\end{equation}
It follows that
\begin{equation}
\left\|
\left(A_{v} \; B_{v}\right)
\right\|_2
\le g_{\max}.\label{eqn:g_max_bound}
\end{equation}
Moreover, by \eqref{eqn:grad_f} and \eqref{eqn:sigma}, we have
\begin{equation}\label{eqn:grad_f_bound}
\|\nabla_{\theta} f(v \vert \theta)\|_2 \le \alpha_{\sigma}g_{\max}.
\end{equation}

Finally, we remark that if the two terms in the sum in \eqref{eqn:f} are weighted
by $\alpha$ and $1-\alpha$ respectively as described in \eqref{eqn:f_weighted} and \eqref{eqn:g_weighted},
then \eqref{eqn:g_max_bound} and \eqref{eqn:grad_f_bound} hold with
\begin{equation}\label{eqn:g_max_weighted}
g_{\max}:=\mu(\tilde{H})\sqrt{\alpha^2 \|X_V\|_2^2\; \mu(\tilde{H})^2 + (1-\alpha)^2 \|X_E\|_2^2}.
\end{equation}

\subsection{The SGD Algorithm}\label{sec:sgd}

Let $\mc{D}$ be an unknown joint distribution of the vertex and the associated label.
Let
$$
\mc{S}=\{(v_1,y_1),(v_2,y_2), \ldots, (v_n, y_n)\}
$$
be a set of $n$ i.i.d. samples from $\mc{D}$.
The set $\mc{S}$ serves as the training set for the HCoN.
The objective function of an HCoN is given by
$$
\mc{L}(\theta) = \frac{1}{n}\sum_{i=1}^{n}\ell(f(v_i \vert \theta), y_i).
$$
The learning task is a transductive semi-supervised task. The incident matrix $H$ and
the feature matrices $X_{V}$ and $X_{E}$ are considered given. The formation of a training set is a sampling
of vertices (and the associated labels) in the network. The training process (or learning process) refers to
the minimization of $\mc{L}$ via SGD. A testing sample is another i.i.d. sample of a vertex in the network.
The testing process refers to taking the output at the testing sample and comparing
it to the known label. The sample space is therefore a finite space, to which the concentration
inequality we used below applies.

We consider $T$ iterations of the SGD algorithm, where the batch size is $1$. At the $t$-th iteration,
a sample $(v_{i_t},y_{i_t})$ is drawn from $\mc{S}$ with replacement.
The parameters $\theta=(Q_V^{\top},Q_E^{\top})^{\top}$ are updated as follows:
$$
\theta_{t} = \theta_{t-1} - \eta \nabla_{\theta}\ell(f(v_{i_t} \vert \theta_{t-1}), y_{i_t})
$$
for $t=1,2,\ldots,T$, where $\eta>0$ is the learning rate.
The final parameters learnt are $\theta_T$, also denoted by $\theta(\mc{S},A)$.
The variable $A$ denotes a particular randomization of SGD, i.e. the sequence $(i_1,i_2,\ldots,i_T)$.

Let $\mc{S}'=\{(v_1',y_1'),(v_2',y_2'), \ldots, (v_n', y_n')\}$
be a training set that differs from $\mc{S}$ with one sample. That is, for some $1\le i^*\le n$,
$$
\left\{
\begin{array}{ll}
(v_{i}', y_{i}') = (v_{i}, y_{i}), & \mbox{if } i\ne i^*\\
(v_{i}', y_{i}') \ne (v_{i}, y_{i}), & \mbox{if } i = i^*.
\end{array}
\right.
$$
Let $\theta'=({Q_V'}^{\top},{Q_E'}^{\top})^{\top}$ be the parameters learned with $\mc{S}'$.
The SGD update is given by
$$
\theta_{t}' = \theta_{t-1}' - \eta \nabla_{\theta}\ell(f(v_{i_t}' \vert \theta_{t-1}'), y_{i_t}')
$$
for $t=1,2,\ldots,T$.
The initial parameters for $\mc{S}$ and $\mc{S}'$ are set equal, i.e. $\theta_0=\theta_0'$.
The parameters learned, using the same randomization $A$ as that for $\mc{S}$, are denoted by
$\theta_T'=\theta'(\mc{S}',A)$.

The difference between the parameters $\theta_t$ and $\theta_t'$ at the $t$-th SGD iteration is
$$
\Delta \theta_t := \theta_t - \theta_t'
= \left[
\begin{array}{c}
Q_{V,t} - Q_{V,t}'\\
Q_{E,t} - Q_{E,t}'
\end{array}
\right].
$$
Since $\mc{S}$ and $\mc{S}'$ are identical except for the $i^*$-th sample,
if $t^*$ is the first iteration at which $(v_{i^*},y_{i^*})$ and
$(v_{i^*}',y_{i^*}')$ are sampled, we have $\Delta \theta_t=0$ for all $t < t^*$.
Note that
\begin{eqnarray}
\Delta \theta_t  &=& \Delta\theta_{t-1}
- \eta \left[\nabla_{\theta}\ell(f(v_{i_t} \vert \theta_{t-1}), y_{i_t})
- \nabla_{\theta}\ell(f(v_{i_t}' \vert \theta_{t-1}'), y_{i_t}')\right].\label{eqn:Delta}
\end{eqnarray}

\section{Main Results}

Following the approach of \cite{hardt-recht-singer15} and \cite{verma-zhang19},
we first establish the uniform stability of SGD and then obtain the generalization guarantees.

\subsection{Uniform Stability}

Recall that $\mc{S}$ and $\mc{S}'$ are two training sets that differ only by the $i^*$-th sample.
In Lemma \ref{lm:1} and Lemma \ref{lm:2}, we bound the gradient difference when the same sample and different samples
are used, respectively. Then, we devise the difference between the SGD updates for
$\mc{S}$ and $\mc{S}'$ in Lemma \ref{lm:3}. The stability then follows in Theorem \ref{thm:1}.
Note that $\sigma'$ is the derivative of $\sigma$. Other variables with a prime ($'$) denote
quantities derived from the perturbed training set $\mc{S}'$.

\begin{lemma}\label{lm:1}
At the $t$-th SGD iteration, we have
\begin{eqnarray*}
\left\|\nabla_{\theta}\ell(f(v \vert \theta_{t-1}), y) - \nabla_{\theta}\ell(f(v \vert \theta_{t-1}'), y)\right\|_2
&\le& (\alpha_{\ell} \nu_{\sigma} + \nu_{\ell}\alpha_{\sigma}^2) g_{\max}^2\|\Delta \theta_{t-1}\|_2.
\end{eqnarray*}
\end{lemma}

\begin{proof} Let $f=f(v \vert \theta_{t-1})$ and $f'=f(v \vert \theta_{t-1}')$.
Let
\begin{eqnarray*}
d &=& A_{v}Q_{V,t-1} + B_{v}Q_{E,t-1}\\
d' &=& A_{v}Q_{V,t-1}' + B_{v}Q_{E,t-1}'.
\end{eqnarray*}
By \eqref{eqn:f_d} and \eqref{eqn:grad_f}, we have $f=\sigma(d)$, $f'=\sigma(d')$, and
\begin{eqnarray*}
\nabla_{\theta}f &=& \sigma'(d) \left(A_{v} \; B_{v}\right)^{\top}\\
\nabla_{\theta}f' &=& \sigma'(d') \left(A_{v} \; B_{v}\right)^{\top}.
\end{eqnarray*}
Note that
$$
d-d' = \left(A_{v} \; B_{v}\right)\Delta\theta_{t-1}.
$$
Therefore, by \eqref{eqn:g_max_bound}, we have
\begin{equation}\label{eqn:d_diff}
 \vert d-d' \vert  \le g_{\max}\|\Delta\theta_{t-1}\|_2.
\end{equation}
Hence,
\begin{eqnarray*}
\left\|\nabla_{\theta}\ell(f, y) - \nabla_{\theta}\ell(f', y)\right\|_2
&=& \left\|\frac{\partial}{\partial \hat{y}}\ell(f, y) \nabla_{\theta}f - \frac{\partial}{\partial \hat{y}}\ell(f', y) \nabla_{\theta}f'\right\|_2\\
&\le& \left\|\frac{\partial}{\partial \hat{y}}\ell(f, y) \nabla_{\theta}f - \frac{\partial}{\partial \hat{y}}\ell(f, y) \nabla_{\theta}f'\right\|_2\\
&& + \left\|\frac{\partial}{\partial \hat{y}}\ell(f, y) \nabla_{\theta}f' - \frac{\partial}{\partial \hat{y}}\ell(f', y) \nabla_{\theta}f'\right\|_2\\
&=& \left \vert \frac{\partial}{\partial \hat{y}}\ell(f, y)\right \vert
\left \vert  \sigma'(d) -  \sigma'(d')\right \vert
\left\|
\left(A_{v} \; B_{v}\right)
\right\|_2\\
&& + \left \vert \frac{\partial}{\partial \hat{y}}\ell(f, y) - \frac{\partial}{\partial \hat{y}}\ell(f', y)\right \vert
\left \vert  \sigma'(d')\right \vert
\left\|
\left(A_{v} \; B_{v}\right)
\right\|_2\\
&\le& \alpha_{\ell} \nu_{\sigma} \vert d-d' \vert g_{\max} + \nu_{\ell} \vert f-f' \vert \alpha_{\sigma}g_{\max}\\
&\le& \alpha_{\ell} \nu_{\sigma} \vert d-d' \vert g_{\max} + \nu_{\ell}\alpha_{\sigma}^2 \vert d-d' \vert g_{\max}\\
&\le& (\alpha_{\ell} \nu_{\sigma} + \nu_{\ell}\alpha_{\sigma}^2) g_{\max}^2\|\Delta \theta_{t-1}\|_2.
\end{eqnarray*}
\end{proof}

\begin{lemma}\label{lm:2}
At the $t$-th SGD iteration, we have
\begin{eqnarray*}
\left\|\nabla_{\theta}\ell(f(v \vert \theta_{t-1}), y) - \nabla_{\theta}\ell(f(v' \vert \theta_{t-1}'), y')\right\|_2
&\le& 2\alpha_{\ell}\alpha_{\sigma}g_{\max}.
\end{eqnarray*}
\end{lemma}

\begin{proof} Let $f=f(v \vert \theta_{t-1})$ and $f'=f(v' \vert \theta_{t-1}')$.
By \eqref{eqn:loss} and \eqref{eqn:grad_f_bound}, we have
\begin{eqnarray*}
\left\|\nabla_{\theta}\ell(f, y) - \nabla_{\theta}\ell(f', y')\right\|_2
&=& \left\|\frac{\partial}{\partial \hat{y}}\ell(f, y) \nabla_{\theta}f - \frac{\partial}{\partial \hat{y}}\ell(f', y') \nabla_{\theta}f'\right\|_2\\
&\le& \left \vert \frac{\partial}{\partial \hat{y}}\ell(f, y)\right \vert  \left\|\nabla_{\theta}f\right\|_2 +
\left \vert \frac{\partial}{\partial \hat{y}}\ell(f', y')\right \vert \left\|\nabla_{\theta}f'\right\|_2\\
&\le& 2\alpha_{\ell}\alpha_{\sigma}g_{\max}.
\end{eqnarray*}
\end{proof}

\begin{lemma}\label{lm:3}
Let $\mc{S}$ and $\mc{S}'$ be two training sets that differ by one sample.
Let $\theta=\theta(\mc{S},A)$ and $\theta' = \theta'(\mc{S}',A)$ be the graph filter parameters of the HCoN models
trained using SGD for $T$ iterations on $\mc{S}$ and $\mc{S}'$, respectively.
Then, the expected difference in the filter parameters is bounded by,
$$
\mbb{E}_A[\|\Delta\theta_T\|_2] \le \frac{\kappa_0}{n},
$$
where
$$
\kappa_0 := \frac{2\alpha_{\ell}\alpha_{\sigma}\left\{\left[1+\eta (\alpha_{\ell} \nu_{\sigma} + \nu_{\ell}\alpha_{\sigma}^2) g_{\max}^2\right]^{T}-1\right\}}
{(\alpha_{\ell} \nu_{\sigma} + \nu_{\ell}\alpha_{\sigma}^2) g_{\max}}.
$$
\end{lemma}

\begin{proof}
Let $1\le t\le T$ and let
$$
\mb{h}=\nabla_{\theta}\ell(f(v \vert \theta_{t-1}), y)
-\nabla_{\theta}\ell(f(v' \vert \theta_{t-1}'), y').
$$
By \eqref{eqn:Delta}, we have
\begin{eqnarray*}
\mbb{E}_A[\|\Delta\theta_{t}\|_2] &\le& \mbb{E}_A[\|\Delta\theta_{t-1}\|_2]
+ \eta \mbb{E}_A[\|\mb{h}\|_2].
\end{eqnarray*}
Note that the probabilities of the two scenarios considered in Lemma \ref{lm:1} and Lemma \ref{lm:2}
are $\frac{n-1}{n}$ and $\frac{1}{n}$, respectively.
By Lemma \ref{lm:1} and Lemma \ref{lm:2}, we have
\begin{eqnarray*}
\mbb{E}_A[\|\mb{h}\|_2] &\le& \frac{n-1}{n}\cdot\mbb{E}_A[C\|\Delta\theta_{t-1}\|_2] + \frac{1}{n}\cdot \mbb{E}_A[C']\\
&\le& C\mbb{E}_A[\|\Delta\theta_{t-1}\|_2] + \frac{1}{n}\cdot C',
\end{eqnarray*}
where $C:=(\alpha_{\ell} \nu_{\sigma} + \nu_{\ell}\alpha_{\sigma}^2) g_{\max}^2$ and
$C':=2\alpha_{\ell}\alpha_{\sigma}g_{\max}$.
Hence, we have
\begin{eqnarray*}
\mbb{E}_A[\|\Delta\theta_{t}\|_2]
&\le& (1+\eta C) \mbb{E}_A[\|\Delta\theta_{t-1}\|_2] +  \frac{\eta C'}{n}.
\end{eqnarray*}
Solving the above recursion with the initial condition $\Delta\theta_0=0$ yields
\begin{eqnarray*}
\mbb{E}_A[\|\Delta\theta_{T}\|_2] &\le& \frac{\eta C'}{n} \cdot \frac{(1+\eta C)^{T}-1}{\eta C}
= \frac{\kappa_0}{n}.
\end{eqnarray*}
\end{proof}

Next, we prove the uniform stability of the single-layer HCoN model trained using the SGD algorithm.
Uniform stability has been introduced in \cite{bousquet-elisseeff02} for the study of nonrandomized algorithms,
where the algorithms are assumed to be insensitive to the order of the training set. However, SGD is a randomized algorithm
that depends on the order of the random samples. In \cite{hardt-recht-singer15}, the notion of uniform stability
was extended to randomized algorithms. A randomized algorithm is said to be \emph{uniformly stable} if
there exists a constant $\beta$, possibly depending on $n$, such that
$$
\sup_{\mc{S},\mc{S}',v,y} \vert \mbb{E}_A[\ell(f(v \vert \theta), y) - \ell(f(v \vert \theta'), y)] \vert  \le \beta.
$$
Here, $\theta=\theta(\mc{S},A)$ and $\theta' = \theta'(\mc{S}',A)$ are the parameters learned with the datasets $\mc{S}$ and $\mc{S}'$, respectively.
Thus the difference in the loss function values is averaged over all randomizations.
In this paper, we adopt this notion of uniform stability to establish a bound of the generalization gap.

\begin{theorem}[Uniform Stability]\label{thm:1}
The single-layer HCoN model trained using the SGD algorithm for $T$ iterations is uniformly stable:
$$
\sup_{\mc{S},\mc{S}',v,y} \vert \mbb{E}_A[\ell(f(v \vert \theta), y) - \ell(f(v \vert \theta'), y)] \vert  \le \frac{\kappa}{n},
$$
where
\begin{eqnarray}
\kappa &:=& \alpha_{\ell}\alpha_{\sigma}g_{\max}\kappa_0 \nonumber\\
&=& \frac{2\alpha_{\ell}^2\alpha_{\sigma}^2\left\{\left[1+\eta (\alpha_{\ell} \nu_{\sigma} + \nu_{\ell}\alpha_{\sigma}^2) g_{\max}^2\right]^{T}-1\right\}}
{\alpha_{\ell} \nu_{\sigma} + \nu_{\ell}\alpha_{\sigma}^2}.\label{eqn:kappa}
\end{eqnarray}
\end{theorem}

\begin{proof}
Let $f=f(v \vert \theta)$ and $f'=f(v \vert \theta')$.
By \eqref{eqn:f_d}, we have $f = \sigma(d)$ and $f' = \sigma(d')$, where
\begin{eqnarray*}
d &=& A_{v}Q_{V} + B_{v}Q_{E}\\
d' &=& A_{v}Q_{V}' + B_{v}Q_{E}'.
\end{eqnarray*}
Hence, by \eqref{eqn:loss}, \eqref{eqn:sigma}, \eqref{eqn:d_diff} and Lemma \ref{lm:3}, we have
\begin{eqnarray*}
\left \vert \mbb{E}_A\left[\ell(f(v \vert \theta), y) - \ell(f(v \vert \theta'), y)\right]\right \vert
&\le& \mbb{E}_A\left[\left \vert \ell(f(v \vert \theta), y) - \ell(f(v \vert \theta'), y)\right \vert \right]\\
&\le& \alpha_{\ell}\mbb{E}_A\left[\left \vert f(v \vert \theta) - f(v \vert \theta')\right \vert \right]\\
&\le& \alpha_{\ell}\alpha_{\sigma}\mbb{E}_A\left[\left \vert d-d'\right \vert \right]\\
&\le& \alpha_{\ell}\alpha_{\sigma}g_{\max}\mbb{E}_A\left[\left\|\Delta \theta_T\right\|_2\right]\\
&\le& \alpha_{\ell}\alpha_{\sigma}g_{\max} \cdot \frac{\kappa_0}{n}
= \frac{\kappa}{n}.
\end{eqnarray*}
\end{proof}

\subsection{Generalization Gap}

Consider an HCoN trained on $\mc{S}$ with a randomization $A$.
Denote the output of the trained network by $f(v \vert \theta(\mc{S},A))$.
Let $z=(v, y)$ be a random sample from $\mc{D}$. We denote the loss w.r.t. $z$ by
$$
\ell(\mc{S}, A, z) := \ell(f(v \vert \theta(\mc{S},A)), y).
$$
The \emph{generalization error} is defined by
$$
R(\mc{S}, A) := \mbb{E}_z[\ell(\mc{S},A,z)].
$$
The \emph{empirical risk} (a.k.a. \emph{training error}) is defined by:
$$
R_{\mathrm{emp}}(\mc{S},A) := \frac{1}{n}\sum_{i=1}^{n}\ell(\mc{S},A,z_i),
$$
where $z_i=(v_i,y_i)$ is the $i$-th sample in $\mc{S}$. The \emph{generalization gap} is given by
$$
G(\mc{S}) = \mbb{E}_A[R(\mc{S}, A) - R_{\mathrm{emp}}(\mc{S},A)].
$$

Next, we present a perturbation result for $G$. It is a general result for stable algorithms.
The proof can be found in the Appendix of \cite{verma-zhang19}. However, we also include the proof here
for completeness and we have fixed a few minor typos of \cite{verma-zhang19}.
Recall that by Theorem \ref{thm:1}, we have the uniform stability
$$
 \vert \mbb{E}_A[\ell(f(v \vert \theta), y) - \ell(f(v \vert \theta'), y)] \vert  \le \frac{\kappa}{n}.
$$

\begin{lemma}[Gap perturbation]\label{lm:4}
Let $\mc{S}$ and $\mc{S}'$ be two training sets that differ by their $i^*$-th samples. Then,
$$
 \vert G(\mc{S}) - G(\mc{S}') \vert  \le \frac{2\kappa + \gamma_{\ell}}{n}.
$$
Here, $\kappa$ is given by \eqref{eqn:kappa} and $\gamma_{\ell}$ is the upper bound of $\ell$.
\end{lemma}

\begin{proof}
First, consider the perturbation of $R(\mc{S},A)$:
\begin{eqnarray*}
\left \vert \mbb{E}_A\left[R(\mc{S},A) - R(\mc{S}',A)\right]\right \vert
&=& \left \vert \mbb{E}_A\left[\mbb{E}_z\left[\ell(\mc{S},A,z) - \ell(\mc{S}',A,z)\right]\right]\right \vert \\
&=& \left \vert \mbb{E}_z\left[\mbb{E}_A\left[\ell(\mc{S},A,z) - \ell(\mc{S}',A,z)\right]\right]\right \vert \\
&\le& \mbb{E}_z\left[\left \vert \mbb{E}_A\left[\ell(\mc{S},A,z) - \ell(\mc{S}',A,z)\right]\right \vert \right]
\le \frac{\kappa}{n}.
\end{eqnarray*}
Second, consider the perturbation of $R_{\mathrm{emp}}(\mc{S},A)$:
\begin{eqnarray*}
\left \vert \mbb{E}_A \left[R_{\mathrm{emp}}(\mc{S},A) - R_{\mathrm{emp}}(\mc{S}',A) \right]\right \vert
&\le& \left \vert \frac{1}{n}\sum_{j\ne i^*} \mbb{E}_A\left[\ell(\mc{S},A,z_j) - \ell(\mc{S}',A,z_j)\right]\right \vert \\
&& + \left \vert \frac{1}{n} \mbb{E}_A\left[\ell(\mc{S},A,z_{i^*}) - \ell(\mc{S}',A,z_{i^*}')\right]\right \vert \\
&\le& \frac{1}{n}\sum_{j\ne i^*} \left \vert \mbb{E}_A\left[\ell(\mc{S},A,z_j) - \ell(\mc{S}',A,z_j)\right]\right \vert \\
&& + \frac{1}{n} \mbb{E}_A\left[\left \vert \ell(\mc{S},A,z_{i^*}) - \ell(\mc{S}',A,z_{i^*}')\right \vert \right]\\
&\le& \frac{n-1}{n}\cdot \frac{\kappa}{n} + \frac{\gamma_{\ell}}{n}
\le \frac{\kappa+\gamma_{\ell}}{n}.
\end{eqnarray*}
Finally, we have
\begin{eqnarray*}
 \vert G(\mc{S}) - G(\mc{S}') \vert  & \le & \left \vert \mbb{E}_A\left[R(\mc{S},A) - R(\mc{S}',A)\right]\right \vert
+ \left \vert \mbb{E}_A \left[R_{\mathrm{emp}}(\mc{S},A) - R_{\mathrm{emp}}(\mc{S}',A) \right]\right \vert,
\end{eqnarray*}
so that the result follows.
\end{proof}

For reference, we state the classical McDiarmid's concentration inequality \cite{mcdiarmid89},
which provides the chance of an inequality with respect to a random training set.

\begin{proposition}[McDiarmid's concentration inequality]\label{prop:1}
If $F(\mc{S})=F(z_1,z_2,\ldots,z_n)$ is a function of $n$ i.i.d. random variables that satisfies
$$
 \vert F(\mc{S}) - F(\mc{S}') \vert  \le c
$$
for all $\mc{S}$ and $\mc{S}'$ that differ by one coordinate, then, for any $\epsilon>0$, we have
$$
P\left(F(\mc{S}) \le \mbb{E}_{\mc{S}}[F(\mc{S})] + \epsilon\right) \ge 1 - e^{-\frac{2\epsilon^2}{nc^2}}.
$$
\end{proposition}

We are now in position to present our main result for the generalization gap.

\begin{theorem}[Generalization Gap]\label{thm:2}
Consider a single-layer HCoN model trained on a dataset $\mc{S}$ using the SGD algorithm for $T$ iterations.
The following expected generalization gap holds for all $0<\delta < 1$, with probability at least $1-\delta$,
$$
\mbb{E}_A[R(\mc{S},A) - R_{emp}(\mc{S},A)] \le \frac{\kappa}{n} +
\frac{2\kappa+\gamma_{\ell}}{\sqrt{n}}\cdot\sqrt{\frac{\ln \frac{1}{\delta}}{2}}.
$$
Here, $\kappa$ is given by \eqref{eqn:kappa} and $\gamma_{\ell}$ is the upper bound of $\ell$.
\end{theorem}

\begin{proof}
Let $F(\mc{S}) = G(\mc{S})$ be the generalization gap. By Lemma \ref{lm:4} and McDiarmid's concentration inequality,
we have
$$
P\left\{G(\mc{S}) \le \mbb{E}_{\mc{S}}[G(\mc{S})] + \epsilon\right\} \ge 1 - e^{-\frac{2\epsilon^2}{nc^2}},
$$
where $c:=\frac{2\kappa+\gamma_{\ell}}{n}$. Let $\delta=e^{-\frac{2\epsilon^2}{nc^2}}$. Note that
$$
\epsilon = c\sqrt{\frac{n\ln \frac{1}{\delta}}{2}}
= \frac{2\kappa+\gamma_{\ell}}{\sqrt{n}} \cdot \sqrt{\frac{\ln \frac{1}{\delta}}{2}}.
$$
Hence, the following holds with a probability of at least $1-\delta$:
\setlength{\arraycolsep}{0.1em}%
\begin{eqnarray*}
\mbb{E}_A[R(\mc{S}, A) - R_{\mathrm{emp}}(\mc{S},A)]
&\le& \mbb{E}_{\mc{S}}[G(\mc{S})]
+ \frac{2\kappa+\gamma_{\ell}}{\sqrt{n}}\cdot\sqrt{\frac{\ln \frac{1}{\delta}}{2}}.
\end{eqnarray*}
\setlength{\arraycolsep}{5pt}%
It remains to be shown that $\mbb{E}_{\mc{S}}[G(\mc{S})] \le \kappa/n$. Note that
\setlength{\arraycolsep}{0.1em}%
\begin{eqnarray*}
\mbb{E}_{\mc{S}}[G(\mc{S})]
&=& \mbb{E}_{\mc{S}}[ \mbb{E}_A[R(\mc{S}, A) - R_{\mathrm{emp}}(\mc{S},A)]]\\
&=& \mbb{E}_{\mc{S}}[ \mbb{E}_A[\mbb{E}_z[\ell(\mc{S}, A, z)]]] - \frac{1}{n}\sum_{j=1}^{n} \mbb{E}_{\mc{S}}[ \mbb{E}_A[\ell(\mc{S},A,z_j)]]\\
&=& \mbb{E}_{\mc{S},z}[ \mbb{E}_A[\ell(\mc{S}, A, z)]] - \mbb{E}_{\mc{S}}[ \mbb{E}_A[\ell(\mc{S},A,z_i)]]\\
&=& \mbb{E}_{\mc{S},z_i'}[ \mbb{E}_A[\ell(\mc{S}, A, z_i')]] - \mbb{E}_{\mc{S},z_i'}[ \mbb{E}_A[\ell(\mc{S}',A,z_i')]]\\
&\le& \mbb{E}_{\mc{S},z_i'}[  \vert \mbb{E}_A[\ell(\mc{S}, A, z_i') - \ell(\mc{S}',A,z_i')] \vert ]
\le \mbb{E}_{\mc{S},z_i'}\left[\frac{\kappa}{n}\right] = \frac{\kappa}{n}.
\end{eqnarray*}
\setlength{\arraycolsep}{5pt}%
The last inequality is due to Theorem \ref{thm:1}.
\end{proof}

Theorem \ref{thm:2} states that as $n\to\infty$, the gap converges to zero at the rate of $O(1/{\sqrt{n}})$,
provided that $\kappa$ in \eqref{eqn:kappa} does not grow with $n$ and the network size ($M$ and $N$).
It boils down to the requirement that the $g_{\max}$ given in \eqref{eqn:g_max} does not grow with $n$, $M$ and $N$.
In general, $g_{\max}$ can grow with the size of the network.
For example, if the entries of $X_V$ and $X_E$ are uniformly distributed on the interval $[0,1]$,
then $\|X_V\|_2 = O(\sqrt{NF_V})$ and $\|X_E\|_2 = O(\sqrt{MF_E})$.
We therefore normalize each column of $X_V$ to a unit vector. Thus, we have
$\|X_{V}\|_2 \le \|X_V\|_F = \sqrt{F_V}$, which is a constant independent of the graph size.
Here, $\|\cdot\|_F$ denotes the Frobenius norm.
Likewise, we also normalize $X_E$ so that $\|X_{E}\|_2 \le \sqrt{F_E}$.
Another factor that may affect the growth of $g_{\max}$ is the incidence matrix $H$.
When $H$ is normalized to $\tilde{H}$, it can be shown that the dominant singular value
$\mu(\tilde{H})$ is bounded above by $1$ (see \cite{wu-yan-ng22}).
In contrast, we have $\mu(H)=O(\sqrt{NM})$. With both kinds of normalization in place, we have
$$
g_{\max}\le \sqrt{F_{V}+F_{E}} = O(1),
$$
a constant independent of the graph size. Likewise, $g_{\max}$ in \eqref{eqn:g_max_weighted} is bounded
by $\sqrt{\alpha^2 F_{V} + (1-\alpha)^2F_{E}} = O(1)$.

It is worthwhile to compare our results with some related works. In the study of stable algorithms in \cite{bousquet-elisseeff02},
the learnt model $\theta$ is assumed to be a \emph{global optimizer} of the loss function, which is assumed to be \emph{convex}.
There are no SGD iterations. As such, the corresponding constant $\kappa$ is independent of $T$.
However, the assumption of global optimality is impractical. The work in \cite{hardt-recht-singer15} considers
regression models with $\theta_1,\ldots,\theta_T$ generated with SGD. The constant $\kappa$ is shown to be $O(T)$.
However, there are also some rather strong convexity assumptions in the model that
are not applicable to neural networks with nonlinearities. The bound in \cite{verma-zhang19} for GCNs has the same kind of
exponential dependence on $T$ as ours. This is due to the lack of convexity in the model. Moreover, SGD does not
guarantee a monotonic reduction of the loss function even when the learning rate is coupled with a line search.
Thus the situation encountered in \cite{verma-zhang19} and in this paper is more sophisticated but more practical.
However, we manage to establish the convergence with respect to the size $n$ of the training set.

The result indicates the consistency between training and testing errors, which is important for the
model to be useful. Of course, it relies on the well-known fundamental assumption that the training and testing sets
are drawn from the same distribution. In Section \ref{sec:experiments}, we will also numerically study
the generalization gap with respect to different parameters.

In the paper by \cite{wu-ng22} and \cite{wu-yan-ng22}, vertex and hyperedge classification is studied.
Experimental results on several benchmark datasets have shown that the performance of the hypergraph collaborative network is
better than that of the baseline methods. Our theoretical results for the stability and
generalization of hypergraph collaborative networks can further confirm their usefulness.

Finally, we would like to remark that our theoretical analysis is also valid for hypergraph convolution networks (HCNs) in
\eqref{vertexiterative}, where propagation of embedding is done on vertices only. The HCoN model reduces
to an HCN when the hyperedge feature matrix $X_{E}$ is set to the zero matrix and
the hyperedge feature dimension $F_{E}$ is set to zero.

\begin{corollary}[Generalization Gap]\label{cor:1}
Consider a single layer HCN model trained on a dataset $\mc{S}$ using the SGD algorithm for $T$ iterations.
The following expected generalization gap holds for all $0<\delta < 1$, with probability at least $1-\delta$,
$$
\mbb{E}_A[R(\mc{S},A) - R_{emp}(\mc{S},A)] \le \frac{\kappa}{n} +
\frac{2\kappa+\gamma_{\ell}}{\sqrt{n}}\cdot\sqrt{\frac{\ln \frac{1}{\delta}}{2}}.
$$
Here, $\kappa$ is given by \eqref{eqn:kappa} and $\gamma_{\ell}$ is the upper bound of $\ell$.
\end{corollary}

\section{Experiments}\label{sec:experiments}

The purpose of this section is to numerically study the behavior of the generalization gap of HCoNs.
We refer the reader to \cite{wu-yan-ng22} for the accuracy of the model.

\subsection{Datasets}
We use the widely used benchmark datasets of Citeseer \cite{citeseer},
Cora \cite{cora} and PubMed \cite{pubmed} for evaluations.
The three networks consist of $1498$, $16313$, and $3840$ vertices, respectively.

Citeseer and PubMed are cocitation datasets. In the hypergraph, each vertex represents a document.
The set of citations of a document form a hyperedge. The vertex features for Citeseer and PubMed
are the bag-of-words vector representations and the term frequency-inverse document
frequency (TF-IDF), respectively.
Cora is a coauthorship dataset. Each vertex represents a document
and each hyperedge represents an author connecting to documents of the author.
The vertex features are the TF-IDF vectors of the documents. A more detailed description of the
feature vector generation process can be found in \cite{wu-yan-ng22}. In each of the three networks,
$30\%$ of the vertices are used as the test set; $30\%$--$70\%$ of the vertices are used as the training set.

\subsection{Effect of the Weight}
Fig. \ref{fig:trainsize} shows the generalization gap as a function of the size of the training set.
The weight $\alpha$ in \eqref{eqn:f_weighted} and \eqref{eqn:g_weighted} is varied to see its effects.
The gap values are the average of $10$ different randomizations to serve as a proxy for the expectation $\mathbb{E}_{A}[\cdot]$.
As the size of the training set increases, the generalization gap decreases.
This is predicted with our theory that the gap decreases as the
size of the training set increases. For a fixed size of the training set,
the gap is smaller as $\alpha$ increases; this indicates that for these datasets,
the vertex features are more important for prediction than the hyperedge features.
\begin{figure*}[!t]
\centering
\subfloat[Citeseer dataset]{\includegraphics[width=2.3in]{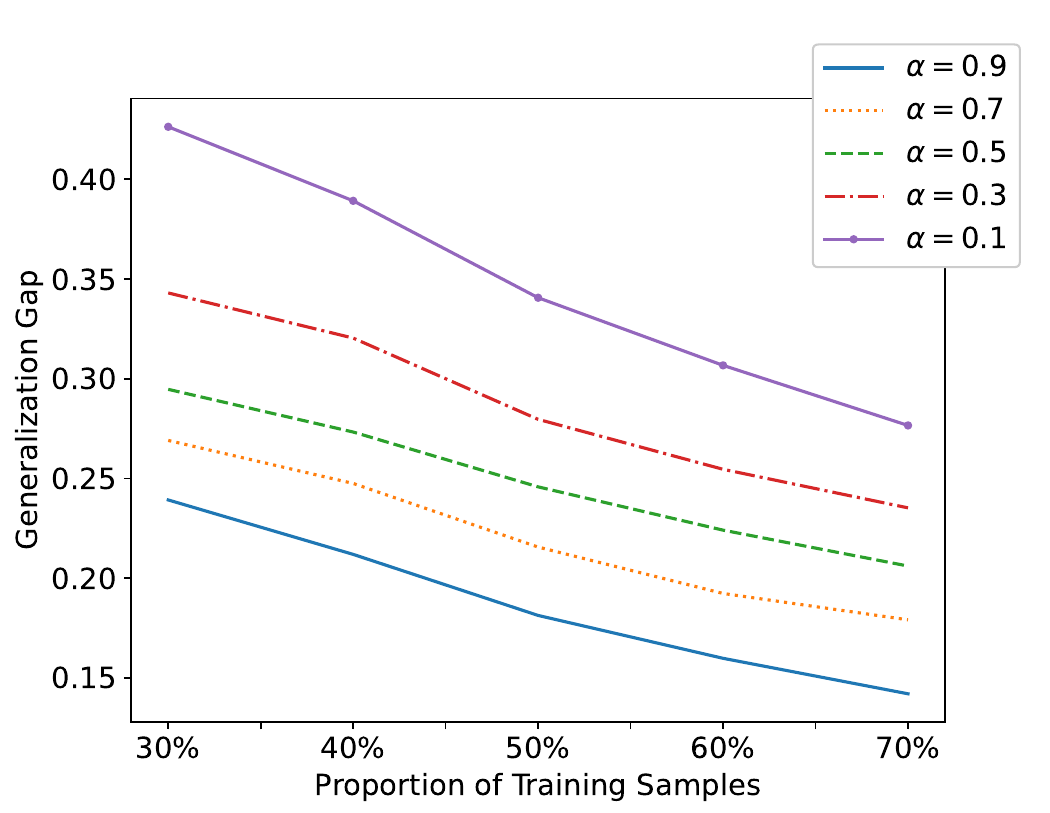}%
\label{fig:trainsize_citeseer}}\\
\subfloat[Cora dataset]{\includegraphics[width=2.3in]{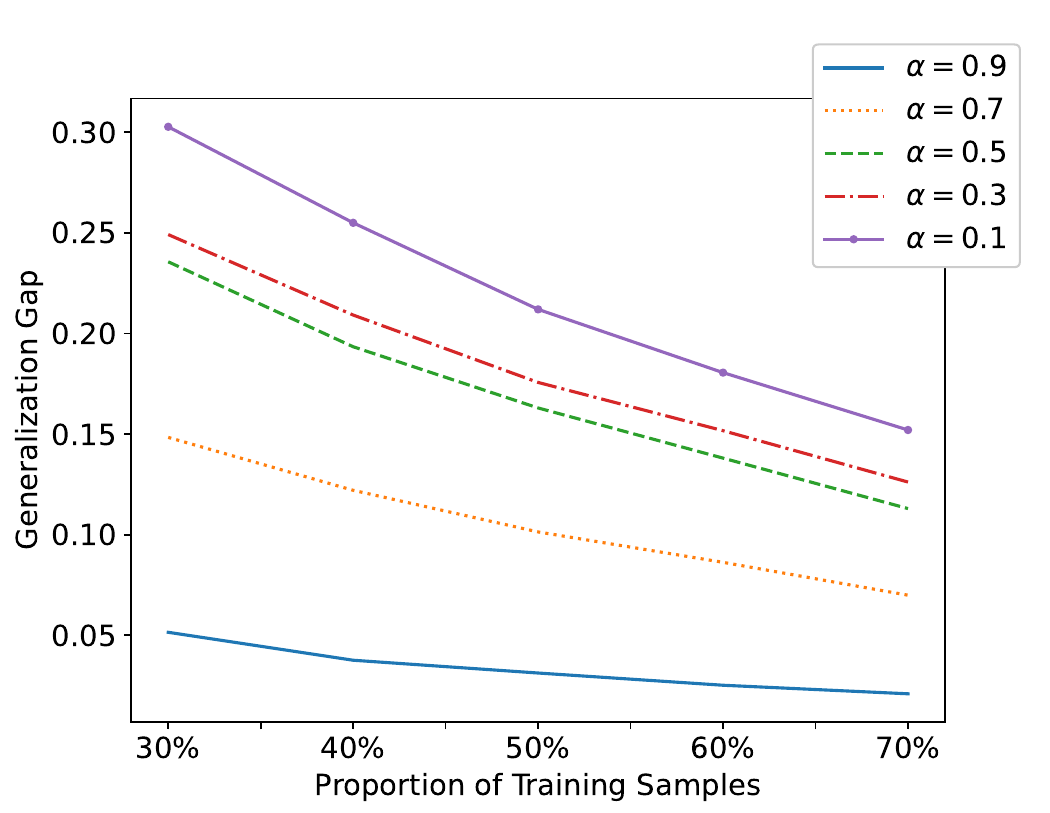}%
\label{fig:trainsize_cora}}\\
\subfloat[PubMed dataset]{\includegraphics[width=2.3in]{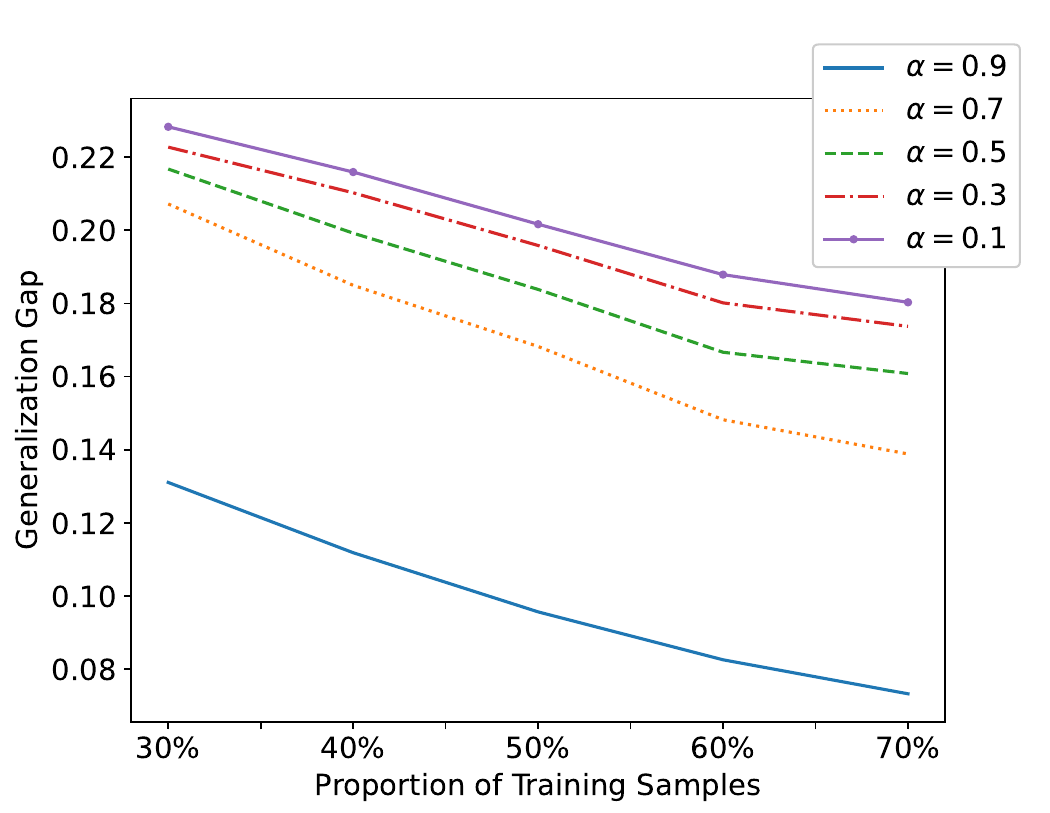}%
\label{fig:trainsize_pubmed}}
\caption{Convergence of generalization gap as the training size $n$ increases.}
\label{fig:trainsize}
\end{figure*}

\subsection{Effect of the Learning Rate}
Fig. \ref{fig:lr} shows the generalization gap as a function of the size of the training set for different learning rates.
The gap values are the average of $10$ different randomizations.
The generalization gap reduces as the size of the training set $n$ increases.
Moreover, the gap decreases with the learning rate for each fixed $n$. Such phenomena are consistent with
the bound devised in Theorem \ref{thm:2}.
\begin{figure*}[!t]
\centering
\subfloat[Citeseer dataset]{\includegraphics[width=2.3in]{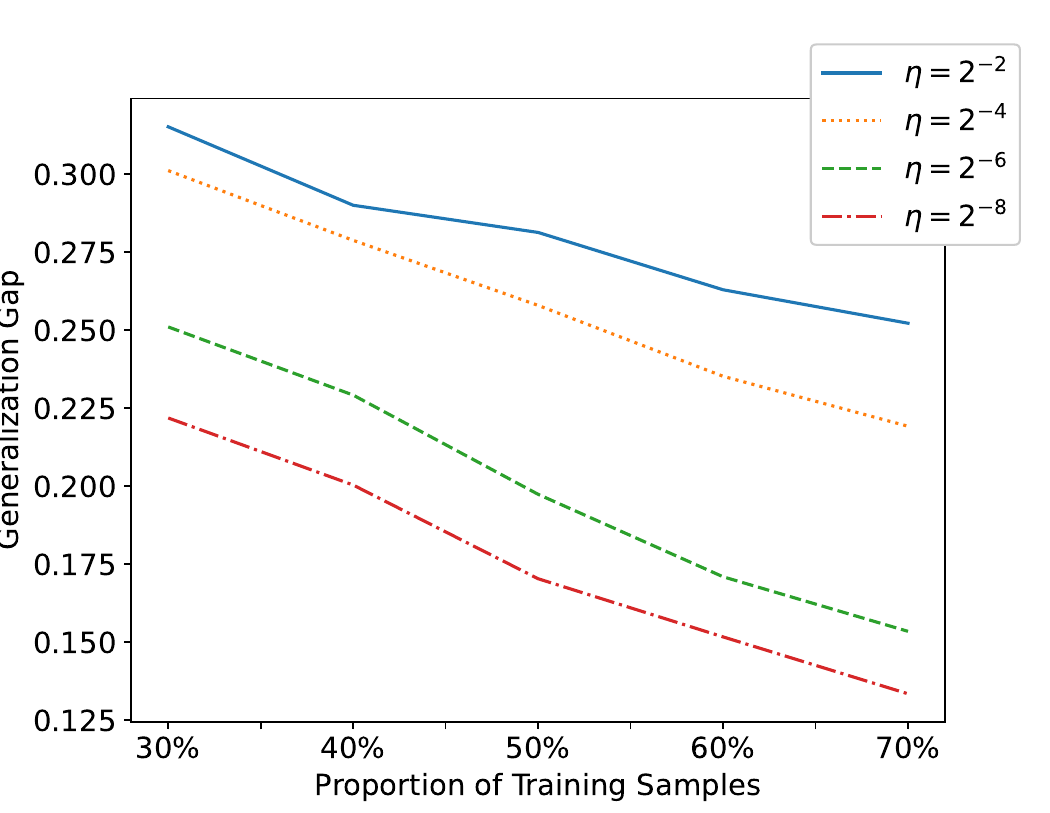}%
\label{fig:lr_citeseer}}\\
\subfloat[Cora dataset]{\includegraphics[width=2.3in]{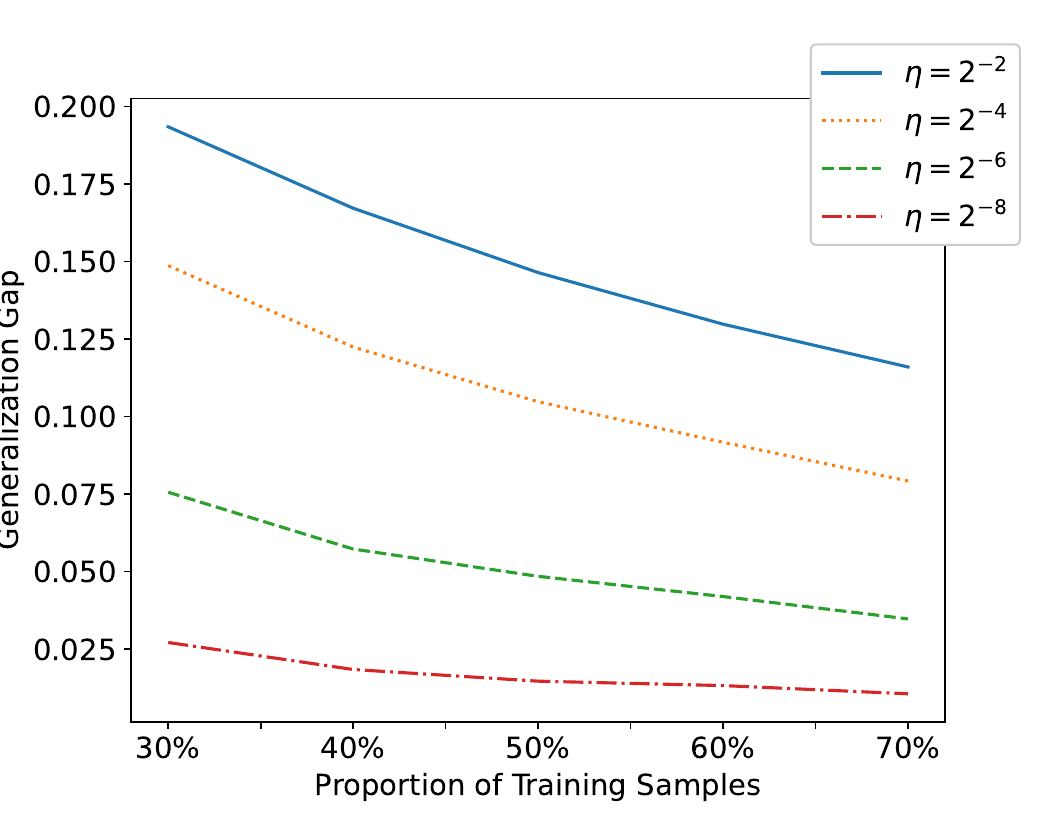}%
\label{fig:lr_cora}}\\
\subfloat[PubMed dataset]{\includegraphics[width=2.3in]{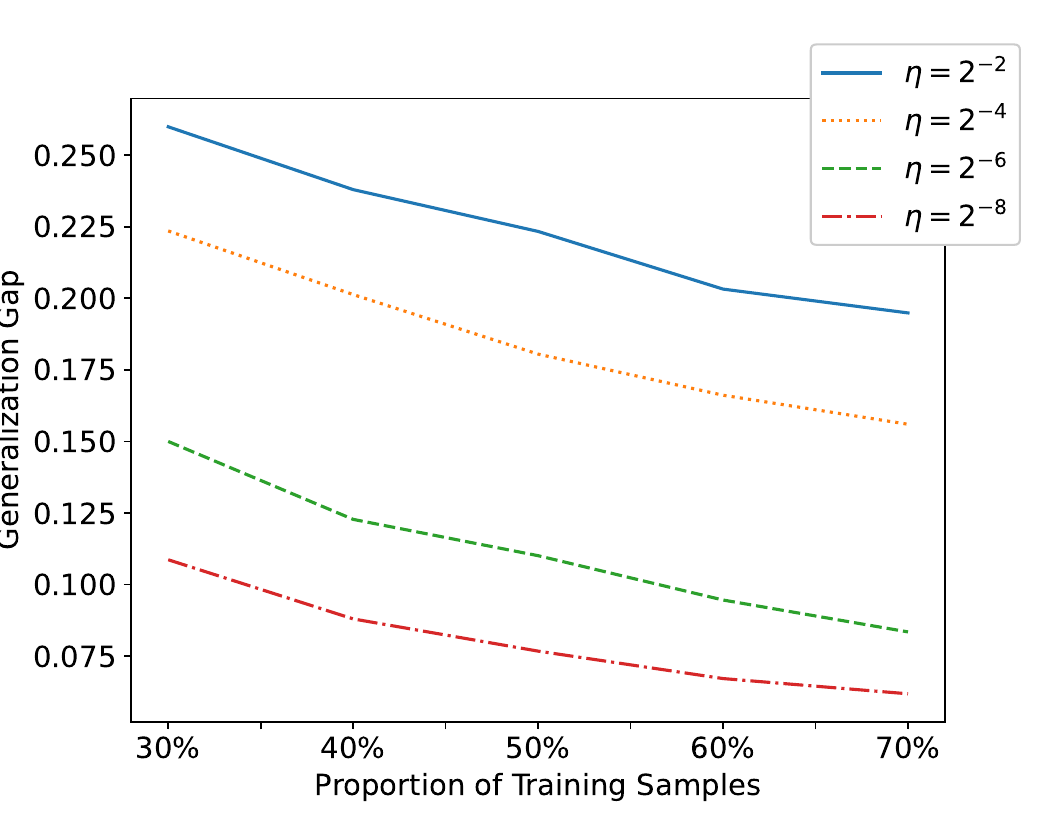}%
\label{fig:lr_pubmed}}
\caption{Convergence of generalization gap as the training size $n$ increases.}
\label{fig:lr}
\end{figure*}

To illustrate the convergence of the training process, we show the training and test loss function values at each epoch in
Fig. \ref{fig:train_test_loss}. To avoid overwhelming with too many figures, we show only the result
for $70\%$ training sample, $\alpha=0.9$ and $\eta=0.01$.
The loss function values are the average of $10$ different randomizations.
The graph shows that 1) the training loss function value is monotonically
decreasing, and hence, the network fits better to the training data; 2) the test loss function value is also monotonically decreasing,
and hence, the generalization error is improving. The starting training and test loss function values are similar because
the network parameters are initialized randomly.
The difference between the training and the test loss function values at the final epoch
constitutes the generalization gap reported in Fig. \ref{fig:trainsize} and Fig. \ref{fig:lr}.
Our theory predicts that the gap at a fixed $T$ decreases as the size $n$ of the training set increases;
see Fig. \ref{fig:trainsize} and Fig. \ref{fig:lr} for the convergence of the gap as a function of $n$.
\begin{figure*}[!t]
\centering
\subfloat[Citeseer dataset]{\includegraphics[width=2.3in]{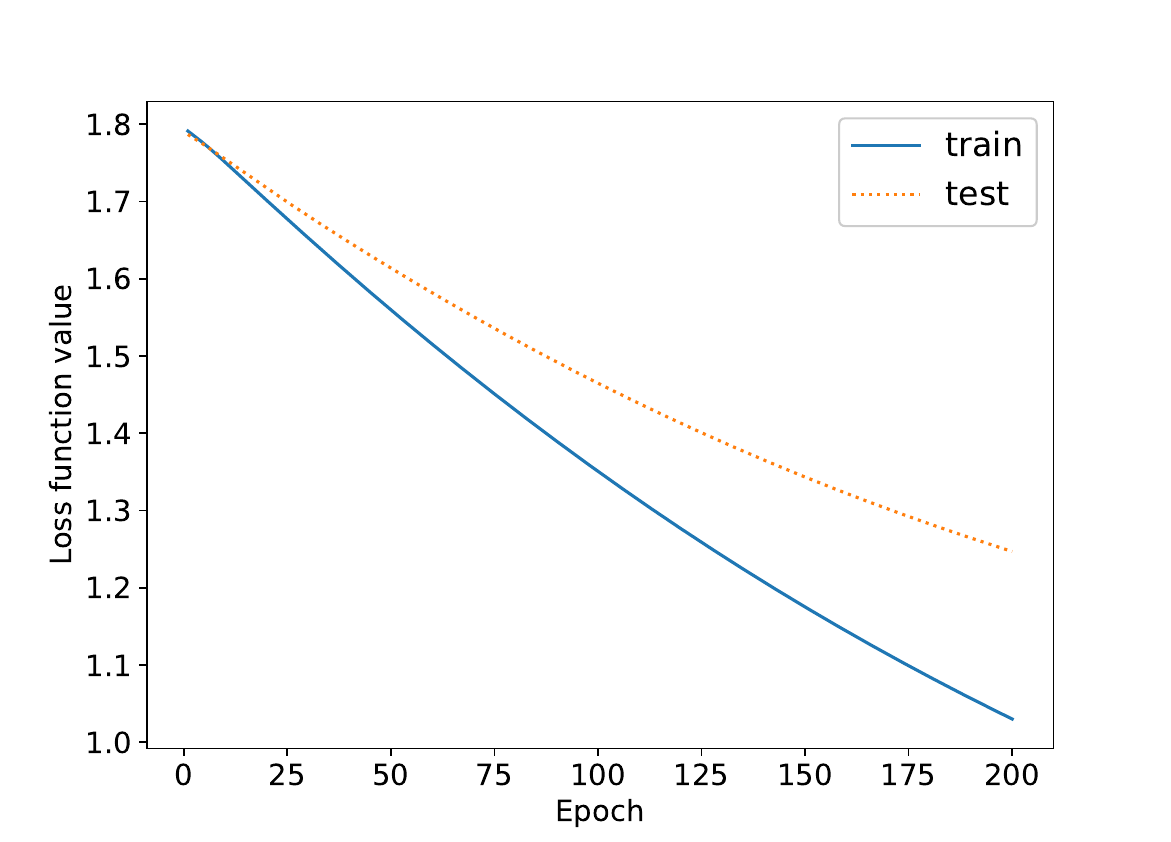}%
\label{fig:train_test_citeseer}}\\
\subfloat[Cora dataset]{\includegraphics[width=2.3in]{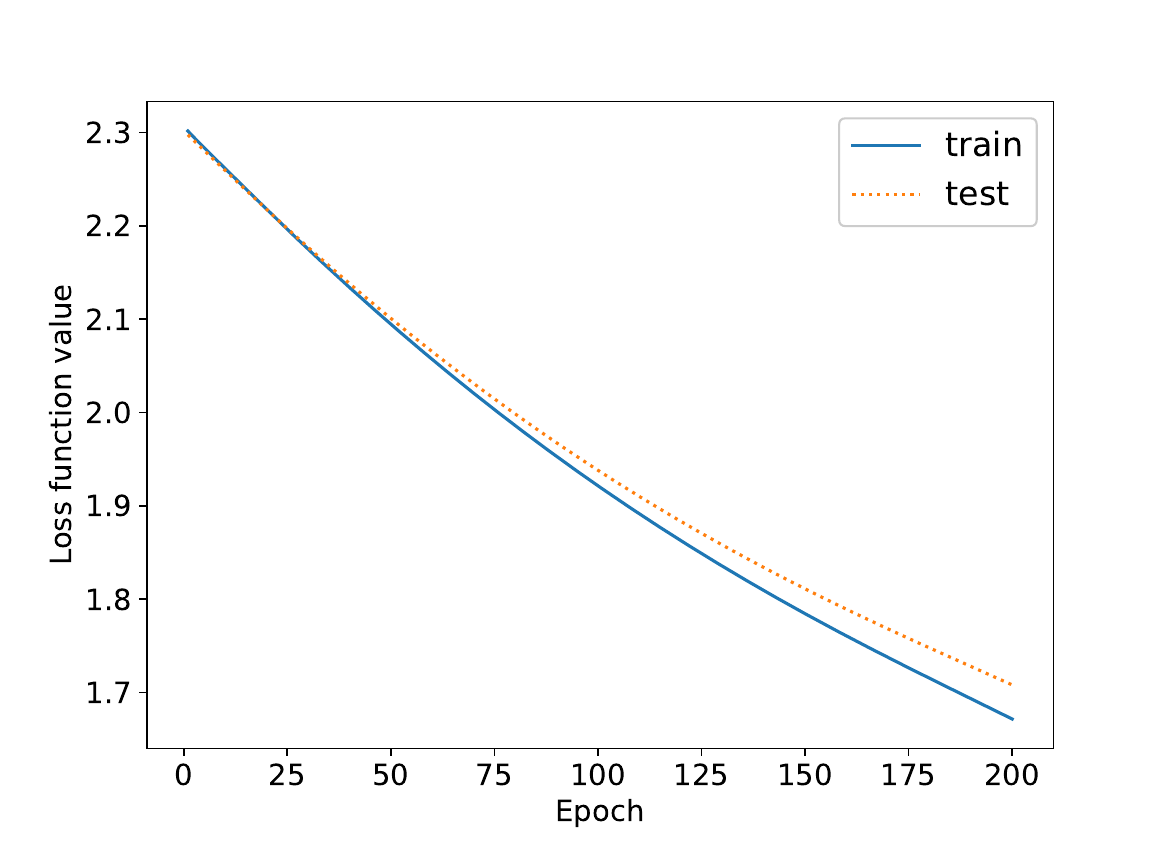}%
\label{fig:train_test_cora}}\\
\subfloat[PubMed dataset]{\includegraphics[width=2.3in]{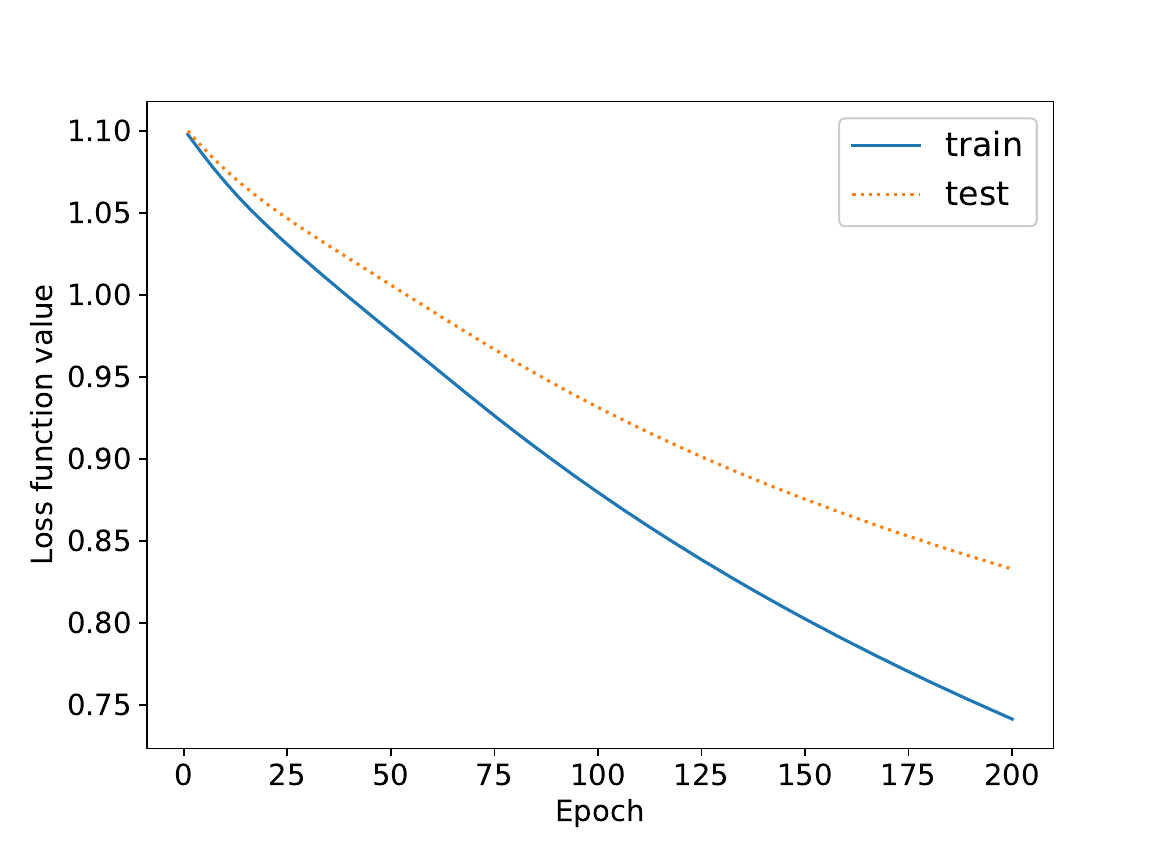}%
\label{fig:train_test_pubmed}}
\caption{Convergence of the training process as the number of iterations $T$ increases.}
\label{fig:train_test_loss}
\end{figure*}

\subsection{Effect of Incidence Matrix Normalization}
In the experiment, we demonstrate the effect of normalization of the incidence matrix $H$.
The normalized model is given in \eqref{eqn:f_weighted} and \eqref{eqn:g_weighted}.
The unnormalized model is obtained by replacing $\tilde{H}$ in \eqref{eqn:f_weighted} and \eqref{eqn:g_weighted} with $H$.
Fig. \ref{fig:epoch} shows the generalization gap as the SGD progresses.
The size of the training set is fixed to $70\%$.
The $\alpha$ is fixed to $0.9$ because it yields the lowest gap and highest accuracy (see \cite{wu-yan-ng22}).
Each epoch represents a cycle of iterations over the whole training set.
The gap values are obtained from a single randomization.
First, the results show that when $H$ is normalized to $\tilde{H}$,
the gap reaches a smaller level and converges faster as the number of iterations increases.
This is predicted with our theoretical bound of the generalization gap since
$\mu(\tilde{H})\le1$ whereas $\mu(H) = O(\sqrt{NM})$.
The theory thus explains why normalization is important.
Second, regarding the results with normalized incidence matrices,
although the theoretical bound grows with $T$ (the number of iterations),
the gap stabilizes quickly in practice.
This is because the bound represents the worst-case scenario, which is pessimistic.
\begin{figure*}[!t]
\centering
\subfloat[Citeseer dataset]{\includegraphics[width=2.3in]{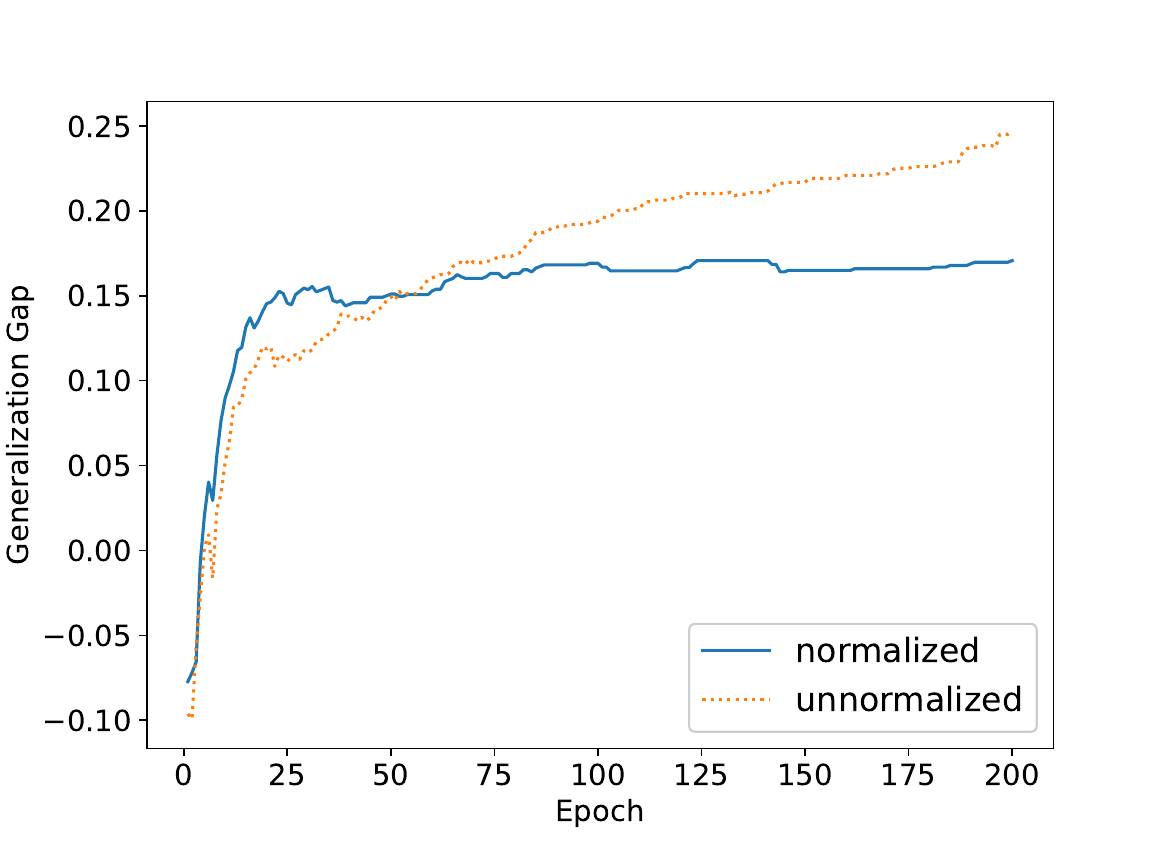}%
\label{fig:epoch_citeseer}}\\
\subfloat[Cora dataset]{\includegraphics[width=2.3in]{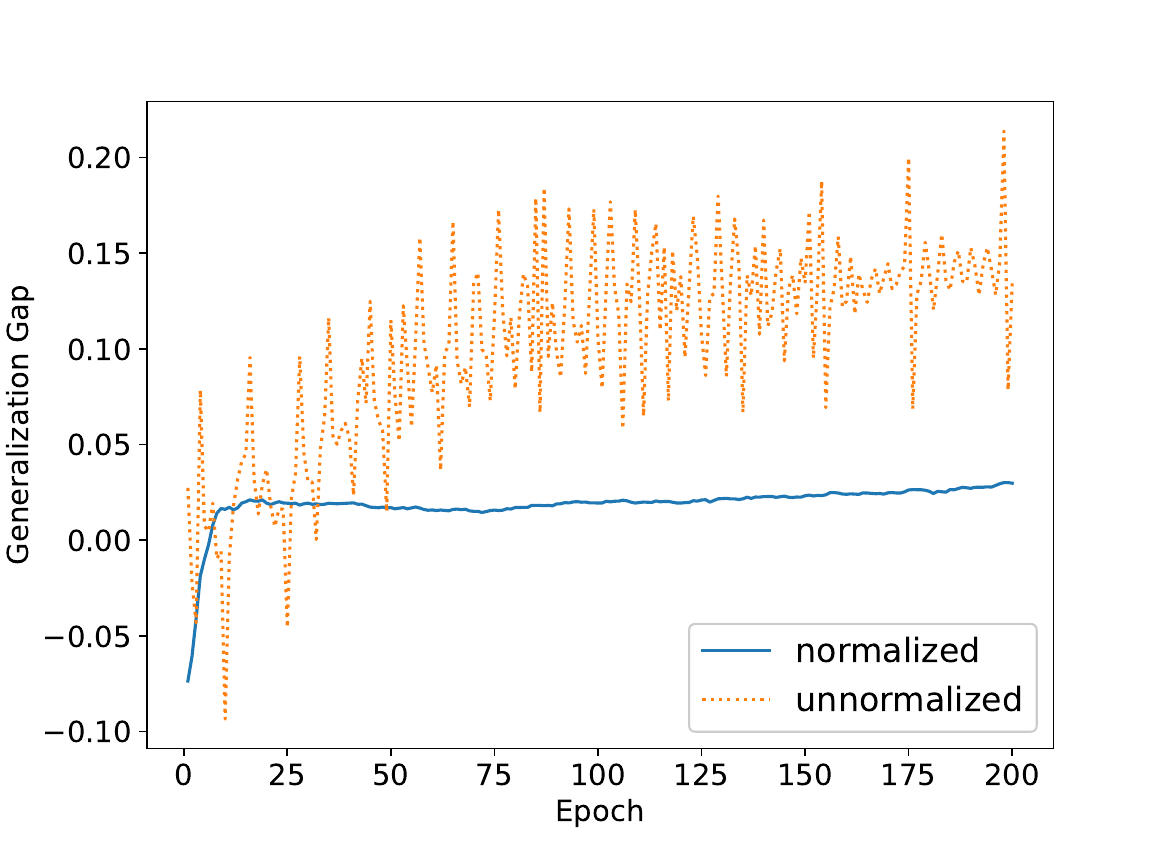}%
\label{fig:epoch_cora}}\\
\subfloat[PubMed dataset]{\includegraphics[width=2.3in]{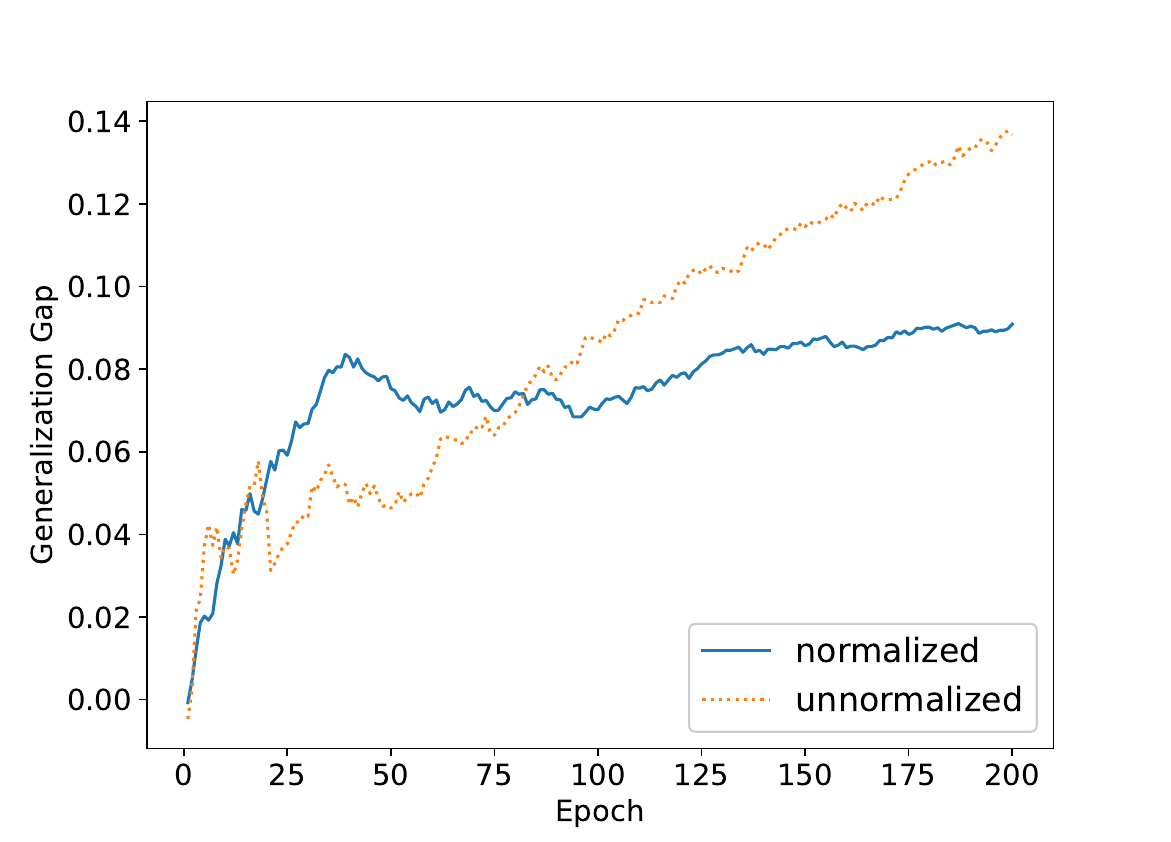}%
\label{fig:epoch_pubmed}}
\caption{Convergence of the generalization gap as the number of iterations $T$ increases.}
\label{fig:epoch}
\end{figure*}

\section{Conclusion}

In this paper, we have established uniform stability and generalization guarantee for
single-layer hypergraph collaborative networks. The results show the importance of the
the normalization of the vertex and hyperedge features and normalization of the hypergraph
incidence matrix $H$. With these normalizations, we have a bound of the generalization gap independent of the
graph size, and therefore, the generalization gap converges to zero. Some numerical experiments
are presented to examine the convergence of the generalization gap in practice.

Several future research directions will be considered. 1) Extend the analysis to multilayer HCoNs,
which are more useful in practice. The main challenge is that the gradients of the network become highly nonlinear
and are more difficult to estimate. 2) Consider more general first-order stochastic optimization algorithms to
include other commonly used algorithms such as SGD with momentum and ADAM. The problem is to devise estimates
that can precisely reflect the potential acceleration delivered by these algorithms and study how the acceleration
affects the stability and generalization gap. 3) While the present result has an advantage in that it does not assume any
specific data distribution, it would also be useful to analyze the generalization gap in the presence of a data distribution.
The idea is to improve the estimation of the expectations $\mathbb{E}_{z\sim \mc{D}}$ and $\mathbb{E}_{\mc{S}}$ in
Lemma \ref{lm:4}, Proposition \ref{prop:1}, and Theorem \ref{thm:2}. Specifically, determine the distribution of the
gap from the assumed distribution of data, and the devise a special case of McDiarmid’s concentration inequality
with an improved bound.

\backmatter

\section*{Acknowledgments}

We would like to thank the editors and reviewers for their comments and suggestions, which helped to improve
the quality of this paper to a great deal.

Ng is supported in part by Hong Kong Research Grant Council GRF
12300218, 12300519, 17201020, 17300021, CRF C1013-21GF, C7004-21GF and Joint NSFC-RGC N-HKU76921.
Wu is supported by Fundamental Research Funds for the Central Universities under Grant 21622326,
National Natural Science Foundation of China (No. 62206111) and China Postdoctoral Science Foundation (No. 2022M721343).

\bibliographystyle{sn-mathphys}
\bibliography{hcon_analysis_mir}

\end{document}